%% file: main.tex
\documentclass[11.7pt,twocolumn,a4paper]{lumia}

\usepackage{multirow}
\usepackage{algorithm}
\usepackage{algorithmic}
\usepackage{subcaption}
\usepackage{calc}
\usepackage{adjustbox}
\usepackage[square,authoryear]{natbib}
\tcbuselibrary{breakable}
\newtheorem{assumption}{Assumption}
\newtheorem{theorem}{Theorem}
\theoremstyle{remark}

\definecolor{darkgreen}{RGB}{0,100,0}
\setheadertext{ }
\title{Delta-K: Boosting Multi-Instance Generation via Cross-Attention Augmentation}
\author[1]{Zitong Wang}
\author[2]{Zijun Shen}
\author[3]{Haohao Xu}
\author[1]{Zhengjie Luo}
\author[1]{Weibin Wu}
\affil[1]{School of Software Engineering, Sun Yat-sen University}

\affil[2]{Nanjing University}
\affil[3]{College of Management and Economics, Tianjin University}

\begin{document}

\begin{abstract}

While Diffusion Models excel in text-to-image synthesis, they frequently suffer from catastrophic concept omission when generating complex multi-instance scenes. Existing training-free methods attempt to resolve this by rescaling attention maps, which merely exacerbates unstructured noise without establishing coherent semantic representations. To address this, we propose Delta-K, a backbone-agnostic, plug-and-play inference framework that  resolves omission by operating directly in the shared cross-attention Key space. Utilizing a lightweight Vision-Language Model (VLM) preview, we isolate a differential key ($\Delta K$) capturing the pure semantic signature of missing concepts, and proactively inject it during the early semantic planning phase. Governed by a dynamically optimized scheduling mechanism, Delta-K grounds diffuse noise into stable structural anchors while naturally preserving existing concepts via the inherent orthogonality of $\Delta K$. Extensive experiments validate its universal applicability, demonstrating that Delta-K significantly improves compositional alignment across both modern DiT and foundational U-Net architectures without requiring spatial masks, auxiliary training, or structural modifications.
\end{abstract}
\maketitle

\section{Introduction}
\label{sec:intro}

The recent progress of large-scale text-to-image diffusion models has substantially advanced open-domain visual synthesis. By scaling data and model capacity, modern latent diffusion models, ranging from U-Net architectures such as SDXL~\citep{podellsdxl} to Diffusion Transformers~\cite{peebles2023scalable, esser2024scaling}, have achieved impressive perceptual quality and generation flexibility. However, as user prompts containing complex compositions, the main bottleneck is shifting from visual realism to faithful semantic grounding. In such case, models still suffer from concept omission and relevant errors~\cite{huang2023t2i}: once a concept fails to be grounded during the early generation process, later denoising steps often struggle to recover it.

Existing training-free methods improve compositional generation through attention guidance, token-level manipulation, prompt decomposition, or spatial constraints~\cite{chefer2023attend,rassin2023linguistic,guo2024initno,li2023gligen,xie2023boxdiff}. However, these methods often rely on predefined interventions and do not explicitly account for the generation state of each concept. More importantly, a weak attention response may reflect not only insufficient emphasis, but also a failure to establish a reliable correspondence between the concept and the emerging visual representation. In this case, reweighting attention maps or imposing spatial constraints may be insufficient. We therefore ask whether present and omitted concepts exhibit different representation-matching behavior during early denoising, and whether this difference can support more targeted intervention.

To investigate this, we trace the attention dynamics of present and omitted concepts throughout sampling. We find that the two groups already differ in the early steps: present concepts tend to form more stable spatial support, whereas omitted concepts show weaker and less consistent responses. On SD3.5, attention statistics from the second denoising step achieve an AUC of about 0.78 for predicting eventual omission. Since cross-attention depends on the similarity ($QK^\top$) between visual queries and textual keys, we interpret omission as an unstable match between the evolving visual representation and the concept key. Supporting this view, repeating the missing concept in the prompt provides only limited gains, while directly modifying its key representation is more effective; adapting the intervention strength to the current attention state further improves over a fixed schedule. These observations motivate early intervention in the key space before the scene structure becomes established.

\begin{figure*}[t]
    \centering  
    \includegraphics[width=1\textwidth]{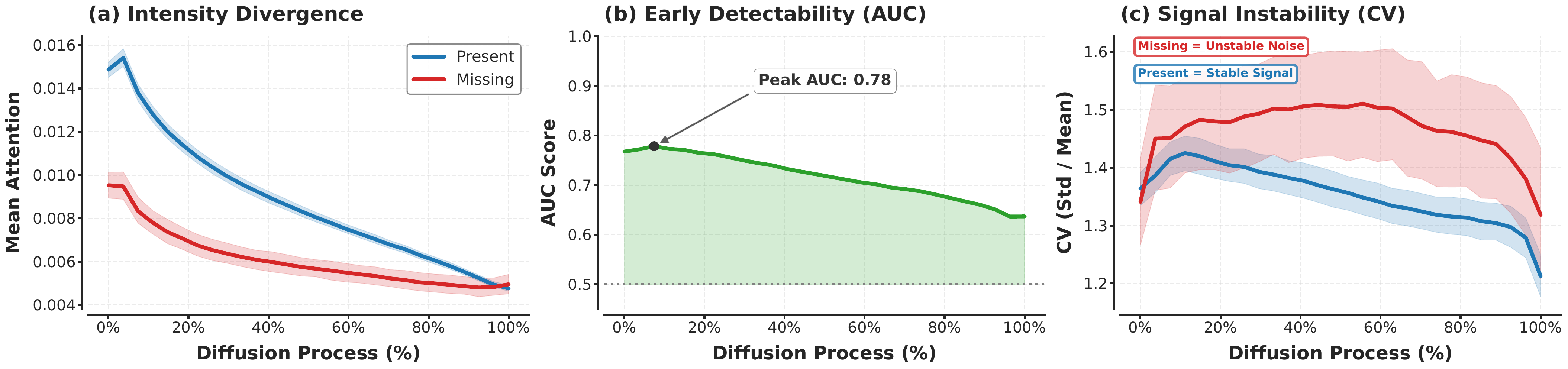} 
\caption{Spatiotemporal dynamics of attention in SD3.5. (a) Missing concepts suffer from chronic intensity suppression but follow valid temporal trends. (b) The high early AUC identifies a semantic planning phase for intervention before image structure solidifies. (c) High instability (CV) characterizes missing tokens as unstable noise.}
    \label{fig:analysis}
\end{figure*}

Based on these observations, we propose \textbf{Delta-K}, a training-free framework that mitigates concept omission through early intervention in the cross-attention key space. Delta-K first performs a short  low-step forward pass and uses a Vision-Language Model (VLM) to identify present and missing concepts. It then replaces each missing concept with \texttt{<MASK>} tokens and computes the difference between the key inputs of the original and masked prompts. This differential representation, denoted by $\Delta K$, captures the key-space contribution of the missing concept and is injected into the corresponding key stream during early denoising. Unlike direct attention rescaling, Delta-K modifies the underlying key representation, requires neither model fine-tuning nor predefined spatial layouts, and applies to both U-Net- and DiT-based diffusion models.

To regulate the intervention, Delta-K optimizes a scalar coefficient $\alpha_t$ at each early denoising step. The attention distribution of successfully generated concepts from the exploratory pass serves as a reference, and $\alpha_t$ is adjusted online to reduce the discrepancy between missing and present concepts.  This adaptive strategy strengthens missing concepts when necessary and limits the intervention once their attention becomes stable.

In summary, our main contributions are as follows:
\begin{itemize}

    \item We show that concept omission is reflected in weak and spatially unstable attention responses during early denoising, suggesting an early representation-matching failure.

    \item We propose {Delta-K}, a training-free framework that detects omitted concepts and constructs a differential key representation for early intervention. 
    \item We introduce an online scheduling strategy that adapts the injection strength to the current attention state. Experiments across  benchmarks and backbones show consistent gains while preserving comparable visual quality.
\end{itemize}

\section{Related Work}
\paragraph{Diffusion Models.} 
Diffusion models~\cite{ho2020denoising, songdenoising} have fundamentally advanced text-to-image synthesis ~\cite{nichol2022glide, ramesh2022hierarchical}. Early dominant frameworks relied on U-Net architectures, notably the Stable Diffusion series up to SDXL~\cite{rombach2022high, podellsdxl}. Recently, inspired by Vision Transformers~\cite{dosovitskiy2020image}, Diffusion Transformers (DiTs)~\cite{peebles2023scalable} have emerged as the new standard, with modern backbones like SD3.5~\cite{esser2024scaling} and FLUX~\cite{flux2024} demonstrating superior scalability. Despite these architectural leaps, accurately grounding multiple instances remains a persistent bottleneck~\cite{jiang2024comat}, frequently leading to severe concept omission in complex scenes~\cite{zhou2025dreamrenderer}. Delta-K directly addresses this vulnerability across both U-Net and DiT paradigms without requiring structural modifications.

\paragraph{Multi-Instance Generation.} 
To mitigate multi-instance generation failures, numerous methods impose structural or semantic guidance. Structure-aware approaches introduce external adapters~\cite{li2023gligen, chen2024training, kim2023dense} or layout priors~\cite{xie2023boxdiff, phung2024grounded, wang2024divide} for composable synthesis. Concurrently, other works inject semantic feedback~\cite{xu2023imagereward, wu2024deep, maexploring} or text-guided reasoning~\cite{yang2024mastering, sun2025dreamsync} during sampling to refine alignment~\cite{ren2024grounded, wen2023improving}. However, fine-grained semantic binding remains fundamentally challenging~\cite{huang2023t2i, yao2025concept}. Methods requiring auxiliary training~\cite{fan2023reinforcement, blacktraining, fang2023boosting} incur substantial computational overhead. Conversely, while training-free attention interventions~\cite{chefer2023attend, rassin2023linguistic, meral2024conform} offer flexibility, they typically rely on post-hoc, heuristic scaling of attention maps~\cite{wang2024tokencompose, agarwal2023star, li2023divide}. This superficial reweighting provides insufficient representational control as prompt complexity increases. Delta-K structurally overcomes this by operating directly in the cross-attention Key space, dynamically injecting omitted semantics at the feature level rather than merely rescaling output attention scores.

\section{Motivation: Dynamics of Concept Omission}

\paragraph{Text-Conditioned Latent Diffusion}
Modern text-to-image models include U-Net architectures such as SDXL~\cite{podellsdxl} and Diffusion Transformers such as SD3.5~\cite{esser2024scaling}. Despite their architectural differences, both perform generation by iteratively updating a noisy latent under textual conditioning. A standard diffusion training objective can be written as
\begin{equation}
    \mathcal{L}
    =
    \mathbb{E}_{z_0,t,\varepsilon}
    \left[
        \left\|
        \varepsilon-\varepsilon_\theta(z_t,t,c)
        \right\|_2^2
    \right],
\end{equation}
where $z_t$ is the noisy latent and $\varepsilon_\theta$ denotes the denoising network conditioned on text $c$.

Textual conditioning is introduced into the visual stream mainly through cross-attention. Given visual queries $Q$ and textual keys $K$ and values $V$, cross-attention is computed as
\begin{equation}
    A=\operatorname{Softmax}
    \left(
        \frac{QK^\top}{\sqrt{d_k}}
    \right)\ \text{and }
    \operatorname{Output}=AV.
\end{equation}
Here, $QK^\top$ determines the compatibility between visual regions and textual concepts. This formulation motivates Delta-K to intervene in $K$, thereby modifying the matching scores before attention normalization rather than only rescaling the resulting attention map.

\label{subsec:window}

To determine when and where concept omission should be corrected, we compare the cross-attention responses of Present and Missing concepts. Unless otherwise stated, all measurements are averaged over all DiT blocks.

\paragraph{Early Intervention Window.}
Previous studies have shown that coarse semantic layouts are mainly established during early denoising~\cite{hertz2022prompt,chefer2023attend}. We examine this temporal behavior specifically to determine the intervention window for concept omission. As shown in Figure~\ref{fig:analysis}(a), missing concepts exhibit weaker attention responses from the beginning of sampling, and this gap persists during subsequent steps. Figure~\ref{fig:analysis}(b) further shows that early attention statistics are predictive of the final omission outcome. We therefore apply Delta-K only during early denoising, before the main scene structure is formed.

\paragraph{Attention Magnitude and Spatial Variation.}
\label{subsec:signal_noise}
We next examine whether concept omission can be explained solely by insufficient attention magnitude. In addition to the mean response, we measure the spatial Coefficient of Variation ($CV$), defined as the ratio of the spatial standard deviation to the mean. As shown in Figure~\ref{fig:analysis}(c), missing concepts exhibit higher $CV$ values than present concepts. Combined with their lower mean attention, this indicates that their responses are both weak and spatially irregular. Therefore, simply increasing attention magnitude may retain the original irregular response pattern rather than establish reliable spatial support.

\paragraph{Motivation for Key-Space Intervention.}
Cross-attention is computed from the compatibility scores \(QK^\top\) between the visual queries and textual keys. We interpret responses of missing concepts as an indication that their key representations do not consistently match the evolving visual queries. This motivates us to intervene in $K$ before attention normalization, changing the underlying matching scores rather than only rescaling the resulting attention map. Based on this motivation, Delta-K constructs a differential key representation $\Delta K$ from the original and masked prompts and injects it into all DiT blocks during the early intervention window.

\section{Methodology}

\begin{figure*}[t]
    \centering  
    \includegraphics[width=\textwidth]{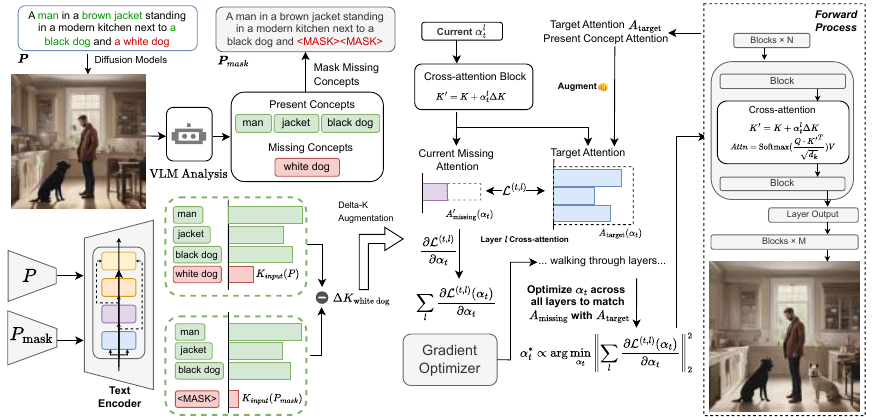}
\caption{Overview of Delta-K. A VLM first separates present and missing concepts from a baseline generation. By
contrasting the original and masked prompts, we obtain a differential key vector $\Delta K$, which is dynamically injected into cross-attention keys during sampling to reinforce missing concepts while preserving existing content.}
\end{figure*}

\subsection{Delta-K}
We propose Delta-K, a training-free method to resolve concept omission in both U-Net and DiT based models by directly intervening in the cross-attention mechanism. The core idea is to compute a differential key vector, $\Delta K$, which isolates the semantic information of missing concepts identified from the text by a VLM. At each inference step, we inject $\Delta K$ into the input of the \texttt{to\_k} module according to a dynamic strength schedule, forcing the model to attend to these previously neglected concepts during image generation.

\paragraph{Identifying Missing Concepts.}
Given a text prompt $P$, we first generate a preliminary image $I_{\text{base}}$. To quickly approximate the semantic outcome of the final generation, we use only one-quarter of the standard number of denoising steps, thereby reducing the overall inference cost of Delta-K. We then employ a vision-language model $\mathcal{F}_{\mathrm{vlm}}$ to identify two sets of concepts:
\begin{equation}
    (C_{\text{present}}, C_{\text{missing}})
    = \mathcal{F}_{\mathrm{vlm}}(P, I_{\text{base}}),
\end{equation}
where $C_{\text{present}}$ denotes the concepts successfully generated in the image and $C_{\text{missing}}$ contains concepts that are missing or incorrectly rendered.

\paragraph{Delta-K augmentation.} First, we construct the masked prompt $P_{\text{mask}}$ by replacing $C_{\text{missing}}$ with \texttt{[MASK]}, and capture the inputs to the \texttt{to\_k} module for both $P_{\text{mask}}$ and the original prompt $P$ during inference to obtain $K_{\text{input}}(P_{\text{mask}})$ and $K_{\text{input}}(P)$. Their difference $\Delta K$ illustrates the semantic representation of the missing concepts:

\begin{equation}
    \Delta K \triangleq K_{\text{input}}(P)-K_{\text{input}}(P_{mask}).
\end{equation}

During the final inference, for each layer and step $t$, we inject this pre-computed $\Delta K$ into the current step's key vector and compute the attention as follows:

\begin{equation}
    \mathrm{Attention}{'}=\mathrm{Softmax}\left(\frac{Q^{} \cdot {(K+\alpha_t\cdot\Delta K)^{\prime}}^T}{\sqrt{d_k}}\right)V,
\end{equation}
where the augmentation strength is controlled by a dynamic scheduling value $\alpha_t$. Since our augmentation is applied only during the early stages of inference, running the corresponding early denoising steps is sufficient to obtain all required $\Delta K$ values. So we can increase the attention weights assigned to missing concepts and promote their generation quality with only modest additional inference overhead.

\subsection{Dynamic scheduling}
Motivated by these observations, we introduce a scheduling method that adaptively adjusts the augmentation strength $\alpha_t$ at each step $t$. The objective is to guide the attention received by the missing concepts $C_{missing}$ to match the attention pattern of successfully generated concepts.

We define $A_{\text{target}}$ the average attention received by the successfully generated concepts $C_{present}$ in baseline generation. At denoising step $t$ and layer $l$, given the baseline attention map $A^{(t,l)}_{\text{base}}$ and the token index set $I_{\text{present}}$, the target attention is computed as:
\begin{equation}A^{(t,l)}_{\text{target}} = \frac{1}{|I_{\text{present}}|} \sum_{i \in I_{\text{present}}} A^{(t,l)}_{\text{base},i}.\end{equation}
Similarly, the target attention distribution $A_{target}^{(t,l)}$ is obtained from the baseline generation, representing the  attention pattern of successfully generated concepts.

\begin{table*}[t]
  \centering
\begin{adjustbox}{max width=\linewidth}

  \begin{tabular}{lcccccc}
    \toprule
    \multicolumn{1}{c}{\multirow{2}{*}{Method}} &
    \multicolumn{3}{c}{Attribute Binding} &
    \multicolumn{2}{c}{Object Relationship} &
    \multirow{2}{*}{Complex$\uparrow$} \\
    \cmidrule(lr){2-4} \cmidrule(lr){5-6}
    & Color$\uparrow$ & Shape$\uparrow$ & Texture$\uparrow$ & Spatial$\uparrow$ & Non-Spatial$\uparrow$ & \\
    \midrule
  TokenCompose~\cite{wang2024tokencompose} & 0.5055 & 0.4852 & 0.5881 & 0.1815 & 0.3173 & 0.2937 \\
    Playground-v2~\cite{playground-v2} & 0.6208 & 0.5087 & 0.6125 & 0.2372 & 0.3098 & 0.3613 \\

    \addlinespace[0.2em]
    \midrule

\color{black!60} SDXL~\cite{podellsdxl}
& \color{black!60} 0.5849
& \color{black!60} 0.4667
& \color{black!60} 0.5279
& \color{black!60} 0.2111
&\color{black!60} 0.3109
& \color{black!60} 0.3230 \\

    +A\&E~\cite{chefer2023attend} & 0.6400 & 0.4517 & 0.5963 & 0.2212 & 0.3142 & 0.3401 \\

    +SynGen~\cite{rassin2023linguistic} & 0.6112 & 0.4831 & 0.5312 & 0.2310 & 0.3249 & 0.3192 \\

    \rowcolor{gray!15}
    +Delta-K (Ours) & 0.6371 & 0.5119 & 0.6218 & 0.2466 & 0.3175 & 0.3532 \\
    \rowcolor{gray!15}
    & \textcolor{darkgreen}{(+0.0522)} & \textcolor{darkgreen}{(+0.0452)} & \textcolor{darkgreen}{(+0.1169)} & \textcolor{darkgreen}{(+0.0355)} & \textcolor{darkgreen}{(+0.0066)} & \textcolor{darkgreen}{(+0.0302)} \\

    \midrule
\color{black!60} SD3.5-M~\cite{esser2024scaling}
& \color{black!60} 0.7939
& \color{black!60} 0.5563
& \color{black!60} 0.7381
& \color{black!60} 0.3053
& \color{black!60} 0.3146
& \color{black!60} 0.3050 \\
    +A\&E & 0.7959 & 0.5612 & 0.7262 & 0.3310 & 0.3221 & 0.3122 \\
    +SynGen & 0.7982 & 0.5662 & 0.7183 & 0.3221 & 0.3129 & 0.3135 \\
    +InitNO & 0.8002 & 0.5591 & 0.7499 & 0.3211 & 0.3240 & 0.3215 \\

    \rowcolor{gray!15}
    +Delta-K (Ours) & 0.8125 & 0.5984 & 0.7816 & 0.3487 & 0.3331 & 0.3410 \\
    \rowcolor{gray!15}
    & \textcolor{darkgreen}{(+0.0186)} & \textcolor{darkgreen}{(+0.0421)} & \textcolor{darkgreen}{(+0.0435)} & \textcolor{darkgreen}{(+0.0434)} & \textcolor{darkgreen}{(+0.0185)} & \textcolor{darkgreen}{(+0.0360)} \\

    \bottomrule
  \end{tabular}
\end{adjustbox}
    \caption{Experiments on T2I-CompBench results. Higher is better ($\uparrow$).}
  \label{tab:t2i_compbench}
\end{table*}
To encourage the attention distribution of missing concepts to match the target distribution, we minimize the following objective:
\begin{equation}
    \mathcal{L}^{(t,l)}(\alpha_t) = \left\| {A_{missing}^{\prime}}^{(t,l)}(\alpha_t) - A_{target}^{(t,l)} \right\|_2^2.
\end{equation}
This objective encourages the model to gradually align the attention allocation of missing concepts with that of successful ones.
Since the loss function is differentiable with respect to the augmentation strength $\alpha_t$, we perform an online optimization at each denoising step $t$. The optimal coefficient is obtained by minimizing the aggregated gradient magnitude across layers:

\begin{equation}
    \alpha_t^* = \lambda \cdot \arg\min_{\alpha_t} \left\| \sum_l \frac{\partial \mathcal{L}^{(t,l)}(\alpha_t)}{\partial \alpha_t} \right\|_2^2.
\end{equation}

In practice, we optimize this objective using Adam~\cite{kingma2014adam} under a fixed computational budget. This dynamic optimization allows the augmentation strength to adapt to the generation process, stabilizing the attention signal of missing concepts and transforming noisy attention patterns into coherent semantic representations. We augment the key and compute cross-attention in a standard form:
\begin{equation}
\label{eq:delta_k_full_attn}
\mathrm{Attn}^{(t,l)}(\alpha_t^*)=
\mathrm{Sm}\!\left(
\frac{
Q^{(t,l)}\left(K^{(t,l)}+\alpha_t^*\Delta K^{(t,l)}\right)^{\top}
}{
\sqrt{d_k}
}
\right)V^{(t,l)}.
\end{equation}

\section{Experiment}
\subsection{Experimental Setup}
In this section, we describe the implementation details, comparison methods, and evaluation benchmarks used to assess Delta-K method proposed above.

\paragraph{Implementation Details.} 
About backbone models, Delta-K is implemented across SDXL~\cite{podellsdxl} and SD3.5-Medium~\cite{esser2024scaling}, covering both U-Net and DiT architectures. Additionally, we further validate Delta-K on a larger FLUX-dev model~\cite{flux2024}, with results provided in Appendix. We choose Qwen3-VL~\cite{Qwen3VL} as the VLM and the augmentation is applied only during the first 10 denoising steps. More details are also provided in Appendix.

\paragraph{Baselines.}
We compare Delta-K against attention- and latent-space-guided methods, including A\&E~\cite{chefer2023attend}, SynGen~\cite{rassin2023linguistic}, and InitNO~\cite{guo2024initno}, as well as recent token-, embedding-, and prompt-based interventions, including ToMe~\cite{hu2024token}, Detail++~\cite{chen2026detail++}, EBOpt, and DOS. We also include Self-Cross~\cite{qiu2025self} and CountGen~\cite{binyamin2025make}, which target instance separation and counting, respectively, to examine whether their benefits extend to mitigating concept omission.

\paragraph{Evaluation Benchmarks.}
We evaluate our method on three benchmarks: T2I-CompBench~\cite{huang2023t2i}, GenEval~\cite{ghosh2023geneval}, and ConceptMix~\cite{wu2024conceptmix}. T2I-CompBench comprises 6,000 prompts covering attribute binding, object relationships, and complex compositions.  We report fine-grained T2I-CompBench results and scores under different baselines in Table~\ref{tab:t2i_compbench}. To provide a comparison on more challenging multi-instance generation, we also report the average ConceptMix score over the harder subsets where $k \in \{5,6,7\}$ in Table~\ref{tab:exp1}. These benchmarks provide a comprehensive assessment of our method's ability to handle multi-object compositions. Full ConceptMix results, GenEval results, and additional aesthetic quality metrics are provided in Appendix.

\begin{table*}[!t]
\centering
\begin{adjustbox}{max width=\linewidth}

\begin{tabular}{lccccccc}
\toprule
& \cellcolor{gray!15}Delta-K & ToMe & Detail++ & EBOpt & DOS & Self-Cross & CountGen \\
\midrule
\multicolumn{8}{l}{\textbf{SD3.5-M}} \\
T2I-CompBench~\cite{huang2023t2i} 
& \cellcolor{gray!15}\textbf{0.5358} & 0.5112 & 0.5319 & 0.5293 & \underline{0.5355} & 0.5316 & 0.5002 \\
ConceptMix~\cite{wu2024conceptmix} 
& \cellcolor{gray!15}\textbf{0.0837} & 0.0792 & 0.0811 & 0.0805 & 0.0822 & \underline{0.0831} & 0.0779 \\
\midrule
\multicolumn{8}{l}{\textbf{SDXL}} \\
T2I-CompBench~\cite{huang2023t2i} 
& \cellcolor{gray!15}\textbf{0.4480} & 0.4392 & 0.4433 & 0.4319 & 0.4441 & \underline{0.4476} & 0.4312 \\
ConceptMix~\cite{wu2024conceptmix} 
& \cellcolor{gray!15}\underline{0.0267} & 0.0232 & \textbf{0.0270} & 0.0225 & 0.0259 & 0.0266 & 0.0210 \\
LAION-AES~\cite{schuhmann2022laion} 
& \cellcolor{gray!15}5.62 & 5.51 & 5.59 & \textbf{5.72} & 5.59 & \underline{5.67} & 5.66 \\
Max Infer Time 
& \cellcolor{gray!15}$\times$1.79 & $\times$2.21 & $\times$1.10 & \textbf{$\times$1.05} & \underline{$\times$1.07} & -- & $\times$6.63 \\
\bottomrule
\end{tabular}
\end{adjustbox}
\caption{Experiments on T2I-CompBench (avg), ConceptMix (avg of $k=5,6,7$), LAION-AES, together with the maximum per-sample inference time reported as method/baseline. Bold and underlined values denote the best and second-best results, respectively. For inference time, lower is better.}
\label{tab:exp1}
\end{table*}

\begin{figure*}
    \centering  
    \includegraphics[width=1\textwidth]{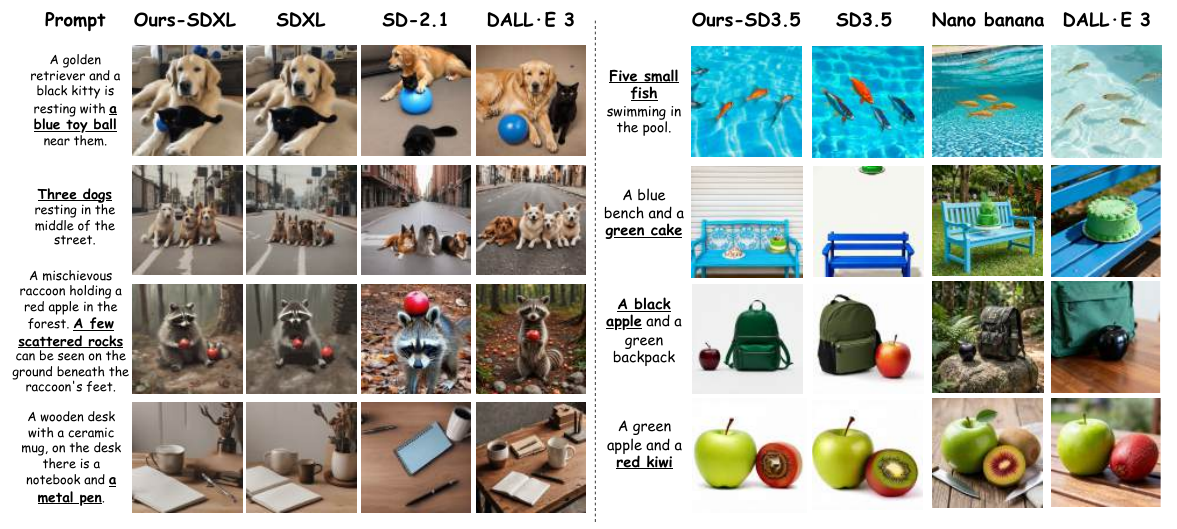} 
    \caption{Visualization. By using SD-2.1, Nano banana and DALL-E 3~\cite{betker2023improving} as baseline methods for comparison, we observe that our approach achieves significant improvements in addressing the instance missing problem.}
    \label{fig:figure2}
\end{figure*}

\subsection{Main Results}
\paragraph{Result Analysis.}
Table~\ref{tab:t2i_compbench} shows that Delta-K consistently improves compositional performance across backbones, outperforming nearly all baseline methods on both SDXL and SD3.5-M. The gains are particularly pronounced for 
\texttt{Texture}, \texttt{Shape} and \texttt{Complex} compositions, which are most closely aligned with Delta-K's objective of improving missing-concept retrieval and multi-entity layout formation. These results validate the effectiveness of injecting differential key signatures during early semantic planning without additional training.

Table~\ref{tab:exp1} further compares Delta-K with recent token, embedding, prompt, and counting-oriented interventions. Delta-K achieves the best average T2I-CompBench score on both SD3.5-M and SDXL, and remains competitive on the harder ConceptMix subsets. These results suggest that directly intervening in key space is more effective for concept omission than generic prompt or token-level modifications.

\paragraph{Efficiency and General Quality.} We further verify that the performance improvement does not come at the cost of inference efficiency or general image quality. As shown in Table~\ref{tab:exp1}, Delta-K maintains comparable aesthetic quality measured by LAION-AES~\cite{schuhmann2022laion}. For efficiency, we report the maximum per-sample inference time on the LAION-AES benchmark. The results show that Delta-K introduces an acceptable inference overhead compared with heavy counting-oriented methods. Additional metrics for generation quality, including CLIPScore~\cite{hessel2021clipscore}, CLIP-IQA+~\cite{wang2023exploring}, and MUSIQ~\cite{ke2021musiq}, are reported in Appendix.

\subsection{Ablation Study}

\begin{table}[t]
\centering
{
\begin{adjustbox}{max width=\linewidth}

\begin{tabular}{llcc}
\toprule
Ablation & Setting & Complex $\uparrow$ & Spatial $\uparrow$ \\
\midrule
\multirow{5}{*}{Schedule}
& Baseline & 0.3230 & 0.2111 \\
& Prompt-only & 0.3303 & 0.2186 \\
& Constant & 0.3326 & 0.2322 \\
& Linear & 0.3402 & 0.2365 \\

&\cellcolor{gray!15}Delta-K & \cellcolor{gray!15}\textbf{0.3532} & \cellcolor{gray!15}\textbf{0.2466} \\
\midrule
\multirow{4}{*}{Steps}
& 5 steps  & 0.3448 & 0.2387 \\
&\cellcolor{gray!15}10 steps & \cellcolor{gray!15}\textbf{0.3532} & \cellcolor{gray!15}\textbf{0.2466} \\
& 30 steps & 0.3491 & 0.2428 \\
& 50 steps & 0.3465 & 0.2404 \\
\midrule
\multirow{4}{*}{VLM}
& GPT-4o & 0.3404 & 0.2352 \\
& Kimi-VL-A3B & 0.3402 & 0.2352 \\
& \cellcolor{gray!15}Qwen3-VL-8B & \cellcolor{gray!15}\textbf{0.3410} & \cellcolor{gray!15}0.2352 \\
& Qwen-VL-Max & 0.3400 & \textbf{0.2355} \\
\bottomrule
\end{tabular}
\end{adjustbox}
}
\caption{Ablation studies. We analyze the augmentation schedule, the number of augmented denoising steps, and the choice of different VLM.}
\label{tab:ablation_all}
\end{table}

We conduct ablation studies on SDXL to analyze the contribution of each component in Delta-K. All variants are evaluated under the same generation setting. 

\paragraph{Scheduling Strategy.}
We study the impact of different augmentation strength scheduling methods. We compare our full method with three simplified alternatives: 
\begin{itemize}
    \item Prompt-only. Appends $C_{missing}$ into the original prompt.
    
    \item Constant-Strength (Constant). We apply fixed augmentation strength $\alpha=0.01$ throughout the diffusion process.  

    \item Linear-Strength (Linear). It adopts a linearly decaying schedule from $\alpha_{\max}$ and gradually decreasing to zero, i.e., $\alpha_t = \alpha_{\max}\times\max(0, 1 - t/T)$.
\end{itemize}
As shown in Table~\ref{tab:ablation_all}, Prompt-only provides only marginal improvement; Constant and Linear schedules perform better, showing that explicitly reinforcing missing concepts is beneficial. However, our adaptive scheduling strategy achieves the best results, suggesting that dynamically adjusting the intervention strength according to the current attention distribution is more effective for stabilizing missing concepts.

\paragraph{Number of Augmented Steps.}
 We apply the augmentation during the first 5, 10, 30, and all 50 denoising steps. As shown in Table~\ref{tab:ablation_all}, applying augmentation in the first 10 steps performs the best. Extending the intervention to later steps provides limited additional benefit and may slightly disturb already formed spatial structures. This supports our observation that concept omission is mainly determined during the early semantic planning stage.

\paragraph{Effect of the VLM Module.} We test several VLMs for concept detection, including GPT-4o, Kimi-VL-A3B-thinking, Qwen3-VL-8B-thinking, and Qwen-VL-Max. As shown in Table~\ref{tab:ablation_all}, different VLMs yield very similar results on both Complex and Spatial metrics. This suggests that Delta-K is robust to the choice of VLM. 

In addition, Delta-K is robust to potential VLM errors. Empirically, we observe that the VLM outputs only concepts from the original prompt, avoiding the introduction of new semantics. The method is also fail-safe: false negatives simply reduce Delta-K to the baseline behavior for the missed concept, while false positives are suppressed by the adaptive scheduler once stable semantic anchors are already formed. We further manually verify 200 complex samples and find that Qwen3-VL achieves high Jaccard similarity with both GPT-4o (0.9317) and human annotations (0.9109), indicating reliable concept detection in our setting.

\begin{figure*}[!t]
    \centering  
    \includegraphics[width=1\textwidth]{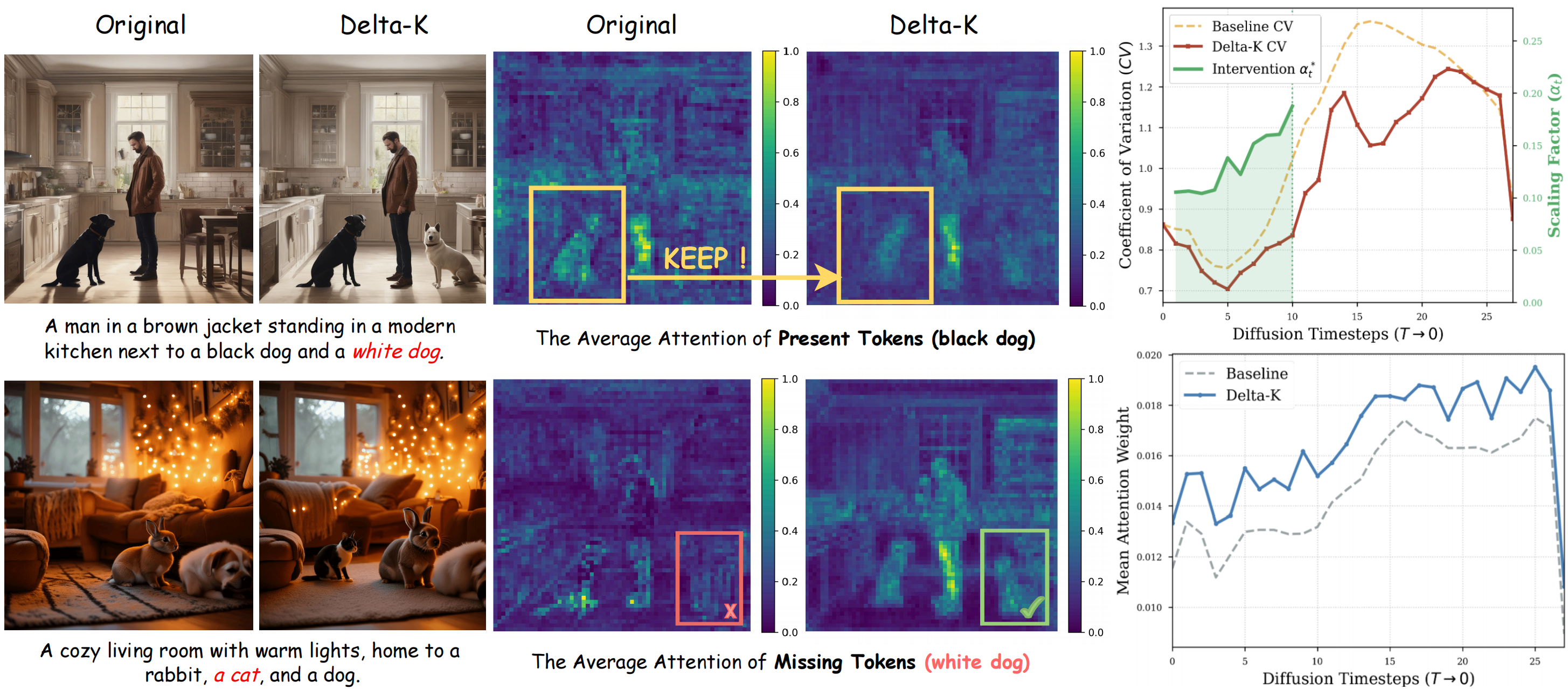} 
    \caption{Case study visualization. Left: Delta-K recovers omitted instances across SDXL and SD3.5; Middle: Cross-attention heatmaps for the SDXL example; Right: Evolution of attention weights (upper) and trajectories of $CV$ and $\alpha_t$ (lower).}
    \label{fig:case_qualitative}
\end{figure*}

\subsection{Case Study: Unpacking Spatiotemporal Dynamics}
\label{subsec:casestudy}

To understand how Delta-K resolves concept omission, we analyze its spatiotemporal dynamics using two challenging prompts: $P_1$ for SDXL (\textit{``A man in a brown jacket standing in a modern kitchen next to a black dog and a white dog.''}) and $P_2$ for SD3.5 (\textit{``A cozy living room with warm lights, home to a rabbit, a cat, and a dog.''}). In both cases, the baseline fails to generate a specific instance (the ``white dog'' and the ``cat''), which our VLM preview correctly identifies.

\paragraph{Qualitative Rectification Across Architectures.} As depicted in Figure~\ref{fig:case_qualitative} (left), the baseline models drop the specified instances. Delta-K successfully synthesizes these missing concepts in both SDXL and SD3.5. And crucially, the recovered concepts integrate naturally without disrupting the global layout or the attributes of pre-existing entities (e.g., the black dog). This confirms the backbone-agnostic robustness of our key-space intervention.

\paragraph{Spatial Grounding: From Noise to Structural Anchor.} 
We visualize the cross-attention heatmaps for $P_1$ in Figure~\ref{fig:case_qualitative} (middle). For the missing concept (``white dog''), the baseline attention is scattered and unanchored. Under Delta-K, this diffuse noise focuses into a highly localized structural anchor. Meanwhile, the attention map for the present concept (``black dog'') remains perfectly unperturbed. This demonstrates that our differential key injection strictly targets the omitted semantics, naturally preserving existing concepts via key-space orthogonality without requiring explicit spatial masks.

We quantitatively track the generation process for $P_2$ in Figure~\ref{fig:case_qualitative} (right) to validate our dynamic scheduling. First, the upper part shows that Delta-K elevates the mean attention activation of the missing tokens, resolving its early suppression. More importantly, the lower part plots spatial instability ($CV$) and our dynamically optimized augmentation strength ($\alpha_t$). While the baseline $CV$ remains stubbornly high (failing to localize), Delta-K concentrates $\alpha_t$ during the early steps, forcing the $CV$ to drop sharply. This proves that our adaptive scheduler does not merely scale up activations, but guides the latent trajectory to establish a stable semantic representation.

\section{Limitations and Future Work}
Our analysis focuses on concept-level attention dynamics, leaving the roles of individual layers less explored. A finer-grained study of information flow may clarify how semantic concepts emerge and propagate during denoising. Another direction is to extend Delta-K into a lightweight trainable framework that predicts more effective augmentation signals and adapts interventions across prompts and model architectures, potentially improving both efficiency and robustness.

\section{Conclusion}
In this work, we revisit the problem of instance omission in text-to-image diffusion models and show that it originates from early-stage semantic mismatches rather than simple attention deficiency. Therefore, we propose Delta-K, a training-free method that intervenes directly in the Key space of cross-attention to reinforce missing concepts during generation. By extracting a differential semantic $\Delta K$ and dynamically injecting it in the early denoising steps, Delta-K encourages missing concepts to receive stable attention and be correctly generated. Extensive experiments across multiple backbones and benchmarks demonstrate that our method consistently improves multi-instance generation while preserving inference efficiency and image quality. These results highlight the effectiveness of early semantic alignment in cross-attention for mitigating concept omission in diffusion models.

\bibliographystyle{plainnat}
\bibliography{references}

\clearpage
\appendix
\counterwithin{figure}{section}
\counterwithin{table}{section}
\numberwithin{equation}{section}
\input{appendix}

\end{document}

%% file: appendix.tex
\section{Experiment details}

\paragraph{Implementation.}
All experiments are implemented in PyTorch and conducted on NVIDIA RTX4090 and A100 GPUs with FP16 precision. We evaluate three representative diffusion backbones: U-Net based Stable Diffusion XL 1.0, DiT-based Stable Diffusion 3.5-M, and Flux-dev. For visual--linguistic alignment, we employ a VLM with temperature set to 0 and strict JSON parsing to extract missing or under-represented concepts from the prompt. The extracted concepts are mapped to token indices using the CLIPTokenizer. Masked text embeddings are then constructed by replacing the corresponding tokens with $<|endoftext|>$ placeholders, which are used to compute the Delta Key values in the cross-attention key augmentation process.
The complete experimental configuration is summarized in Table~\ref{tab:full_config}. We also show our prompt for vlm in Appendix~\ref{appendixb}.

\begin{table*}[t]
\centering
\caption{Experiment details with hardware, VLM config, denoising model configuration, and Delta-K schedule.}
\label{tab:full_config}
\small
\setlength{\tabcolsep}{5pt}
\renewcommand{\arraystretch}{1.1}
\begin{adjustbox}{max width=\linewidth}

\begin{tabular}{p{0.28\linewidth} p{0.22\linewidth} p{0.22\linewidth} p{0.22\linewidth}}
\toprule
\textbf{Category} & \textbf{SDXL} & \textbf{SD3.5} & \textbf{Flux} \\
\midrule

\textbf{Hardware} & & & \\
Framework & PyTorch & PyTorch & PyTorch \\
GPU & RTX4090 / A100 & A100 & A100 \\
Precision & FP16 & FP16 & FP16 \\
\midrule

\textbf{VLM} & & & \\
Temperature & 0 & 0 & 0 \\
Output format & JSON parsing & JSON parsing & JSON parsing \\
\midrule

\textbf{Denoising Model} & & & \\
Tokenizer & CLIP & CLIP \& T5 & T5 \\
Special placeholder & \texttt{<|endoftext|>} & \texttt{<pad>} \& \texttt{<|endoftext|>} & \texttt{<pad>} \\
Total denoising steps & 40 & 28 & 28 \\
\midrule

\textbf{Delta-K Schedule} & & & \\
$\alpha_t^{\max}$ & 0.03 & 3 & 3 \\
Learning rate $\eta$ & 0.001 & 0.002 & 0.002 \\
Iterations per step & 100 & 100 & 100 \\
Augmentation stage & down\_blocks\_1 & transformer\_blocks & transformer\_blocks \\

\bottomrule
\end{tabular}
\end{adjustbox}
\end{table*}

\paragraph{Attention Entropy Analysis.}
To determine the most suitable stage for augmentation, we analyze the stage-level attention entropy during the denoising process using prompt samples such as ``\textit{a living room with a sofa, a man, a brown table and a desk}''.

\begin{figure}[h]
    \centering
    \includegraphics[width=\linewidth]{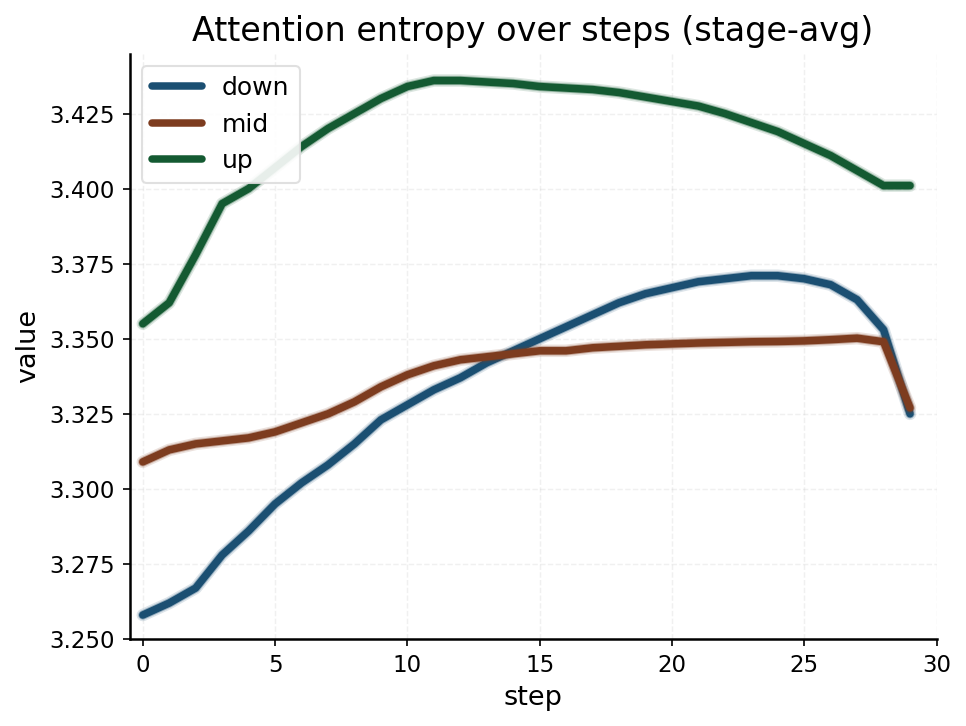}
    \caption{Attention entropy over denoising steps.}
    \label{fig:attention}
\end{figure}

We first study the evolution of attention dispersion across denoising steps. Intuitively, lower entropy indicates that attention is concentrated on a smaller set of key tokens, while higher entropy suggests a more diffuse allocation pattern. This helps reveal which stage is more responsible for establishing global semantic layout and which stage mainly refines local details.

Here $\mathrm{attn}_n(q,k)$ denotes the attention weight between the $q$-th query and $k$-th key in the $n$-th head, and $\mathrm{mass}_n$ represents the aggregated attention mass of each key token summed across all queries. This metric captures the dispersion--concentration dynamics of cross-attention over the denoising process. As shown in Fig.~\ref{fig:attention}, the Down stage exhibits noticeably lower entropy during early steps, indicating concentrated attention patterns responsible for establishing coarse semantic structure, whereas the Up stage maintains higher entropy and is more associated with detail refinement.

The entropy is defined as:
\begin{equation}
H_s(t) = \mathbb{E}_n \left[-\sum P_n \log P_n \right],
\end{equation}

where

\begin{equation}
P_n = \frac{\mathrm{mass}_n}{\sum_k \mathrm{mass}_n[k]}, \qquad
\mathrm{mass}_n = \sum_q \mathrm{attn}_n(q,k).
\end{equation}

Based on this observation, we apply the augmentation to \texttt{down\_block.1} of the U-Net architecture ($32\times32$ resolution) and restrict the intervention to the first 10 denoising steps (the same 10-step augmentation window is used for all three backbones in our implementation). To ensure fair comparisons, all scheduling strategies operate within the same active window. The linear baseline schedules decay linearly from their peak strength, while the proposed Delta-K schedule dynamically adjusts the augmentation coefficients through test-time optimization.

\section{Theoretical Analysis of Delta-K}
\label{sec:theoretical_analysis}

In this section, we provide formal mathematical justifications for the empirical behaviors observed in the main text. Specifically, we analyze why Delta-K does not significantly interfere with successfully generated concepts and how it promotes the spatial concentration of diffuse attention maps. 

\begin{assumption}[Local Intervention]
\label{assum:local_intervention}
Throughout our analysis, we isolate the effect of Delta-K within a single cross-attention block. We assume that for a given forward pass, the visual query vectors $Q$ are fixed, and we analyze the immediate perturbation caused by injecting $\Delta K$. We omit the compounding non-linear effects of Layer Normalization and Feed-Forward Networks across consecutive layers, as our objective is to characterize the local perturbation behavior in the $QK^T$ semantic matching space.
\end{assumption}

\begin{figure*}[t]
    \centering
    \includegraphics[width=\linewidth]{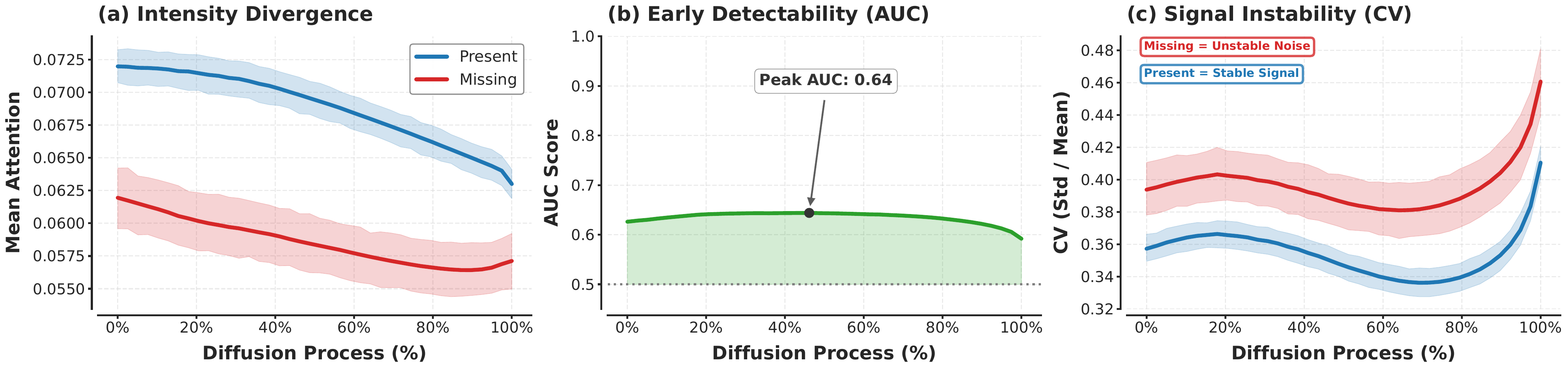} 
    \caption{Spatiotemporal analysis of cross-attention dynamics in SDXL. (a) Intensity Divergence: Mean attention scores for ``Present'' vs. ``Missing'' concepts. (b) Early Detectability: AUC score for detecting omission; the curve remains relatively stable across steps with a peak AUC of 0.64. (c) Signal Stability: Coefficient of Variation (CV) showing that missing concepts correspond to unstable, high-variance attention patterns.}
    \label{fig:analysis_sdxl}
\end{figure*}

\subsection{Semantic Orthogonality and Non-Interference}
\label{sec:orthogonality}

A key advantage of Delta-K is that it enhances omitted concepts while preserving already established semantic bindings. This behavior can be understood through the concentration properties of high-dimensional representations.

\begin{theorem}[Semantic Orthogonality Bound]
\label{thm:orthogonality}
Let the cross-attention dimension be $d_k$. Let $Q_\text{present} \in \mathbb{R}^{d_k}$ denote the query vector associated with a successfully generated concept, and $\Delta K \in \mathbb{R}^{d_k}$ denote the differential key vector extracted through the VLM-guided masking procedure. Assume that embedding vectors follow a sub-Gaussian distribution with variance proxy $\sigma^2$. Under an augmentation strength $\alpha_t$, the probability that the induced perturbation $\delta$ on the pre-softmax attention logit exceeds a threshold $\epsilon>0$ satisfies:
\begin{equation}
\mathbb{P}\left(|\delta|\ge\epsilon\right)
\le
2\exp\left(
-\frac{c\epsilon^2 d_k}{\alpha_t^2\|Q_\text{present}\|_2^2\|\Delta K\|_2^2}
\right),
\end{equation}
where $c>0$ is a constant determined by the sub-Gaussian norm of the embeddings.
\end{theorem}

\begin{proof}
In the baseline generation process, the attention logit associated with a present concept is
\begin{equation}
S_\text{base} =
\frac{1}{\sqrt{d_k}}
\langle Q_\text{present}, K_\text{present}\rangle.
\end{equation}

Under Delta-K augmentation, the key vector becomes
\begin{equation}
K' = K + \alpha_t\Delta K.
\end{equation}

The updated logit therefore becomes
\begin{equation}
S_\text{new}
=
\frac{\langle Q_\text{present},K+\alpha_t\Delta K\rangle}{\sqrt{d_k}}
=
S_\text{base}
+
\underbrace{
\frac{\alpha_t}{\sqrt{d_k}}
\langle Q_\text{present},\Delta K\rangle
}_{\delta}.
\end{equation}

The perturbation $\delta$ depends on the inner product between $Q_\text{present}$ and $\Delta K$. Since $\Delta K$ captures the residual semantic component introduced by the missing concept relative to the masked prompt, its direction is typically weakly correlated with query vectors associated with already established concepts. 

Under the sub-Gaussian embedding assumption, the inner product $\langle Q_\text{present},\Delta K\rangle$ can be treated as a sum of sub-Gaussian random variables. To express the concentration bound compactly, we define the variance proxy
\begin{equation}
\sigma_t^2
=
\frac{\alpha_t^2}{d_k}
\|Q_\text{present}\|_2^2
\|\Delta K\|_2^2.
\label{eq:variance-proxy}
\end{equation}
Applying a standard concentration inequality for inner products of
sub-Gaussian vectors then yields
\begin{equation}
\mathbb{P}\left(
\frac{\alpha_t}{\sqrt{d_k}}
|\langle Q_\text{present},\Delta K\rangle|
\ge \epsilon
\right)
\le
2\exp\left(-\frac{c\epsilon^2}{\sigma_t^2}\right).
\label{eq:concentration-bound}
\end{equation}

Hence the perturbation magnitude decays exponentially with the embedding dimension $d_k$, implying that the influence on already established concepts is negligible with high probability.
\end{proof}

\textbf{Remark.} Theorem~\ref{thm:orthogonality} indicates that the perturbation introduced by $\Delta K$ rapidly diminishes as the representation dimension increases. Combined with the bounded augmentation strength imposed by the dynamic scheduler (e.g., $\alpha_t^{max}=0.03$), this ensures that Delta-K primarily influences queries aligned with the missing concept while leaving existing concepts largely unaffected.

\subsection{Attention Focusing and Spatial Stabilization}
\label{sec:cv_attenuation}

In Section~3.2, we observed that missing concepts typically exhibit diffuse and unstable attention patterns. We now analyze how Delta-K reshapes this distribution by redistributing attention mass toward semantically aligned regions.

\begin{theorem}[Attention Mass Concentration]
\label{thm:cv}
Let $A\in\mathbb{R}^N$ denote the normalized attention distribution of a missing concept across $N$ spatial tokens. Let $\mathcal{I}_\text{target}$ denote spatial locations corresponding to the correct concept region. Under Assumption~\ref{assum:local_intervention}, suppose that injecting $\alpha_t\Delta K$ produces a positive logit shift $\Delta s>0$ for tokens in $\mathcal{I}_{target}$ while leaving background logits approximately unchanged. Then the updated attention distribution $A'$ satisfies:
\begin{equation}
P'_\text{target} > P_\text{target},
\end{equation}
where
\[
P_\text{target}=\sum_{i\in\mathcal{I}_\text{target}}A_i.
\]
\end{theorem}

\begin{proof}
The original attention probability for token $i$ is:
\begin{equation}
A_i=\frac{\exp(s_i)}{\sum_k\exp(s_k)}.
\end{equation}

After Delta-K injection, logits within the target region become:
\begin{equation}
s_i' = s_i + \Delta s_i,
\end{equation}
where:
\begin{equation}
\Delta s_i =
\frac{\alpha_t}{\sqrt{d_k}}
\langle Q_i,\Delta K\rangle.
\end{equation}

For tokens in $\mathcal{I}_\text{target}$, semantic alignment implies $\Delta s>0$. Background tokens receive negligible perturbation due to the orthogonality property established in Theorem~\ref{thm:orthogonality}.

Let:
\begin{equation}
P_\text{target}=\sum_{i\in\mathcal{I}_\text{target}}A_i.
\end{equation}

After the logit shift, the new probability mass assigned to the target region becomes:
\begin{equation}
A'_{i}=
\frac{A_{i}\exp(\Delta s_i)}
{A_{i}\exp(\Delta s_i)+(1-A_{i})}.
\end{equation}

Since $\forall i \in \mathcal{I}_\text{target},\ \exp(\Delta s_i)>1$, it follows that:
\begin{equation}
P'_\text{target}=\sum_{i\in\mathcal{I}_\text{target}}A'_i>P_\text{target}.
\end{equation}

Consequently, attention mass is redistributed from diffuse background regions toward semantically aligned spatial locations. This reallocation increases the concentration of attention around the correct concept region and suppresses background noise.
\end{proof}

\textbf{Remark.} Theorem~\ref{thm:cv} shows that Delta-K increases the attention mass assigned to semantically aligned regions through a positive logit shift. Empirically, this redistribution often corresponds to more stable and spatially localized attention maps, which explains the reduction of instability indicators such as the Coefficient of Variation observed in our experiments.

\begin{table*}[htbp]
\centering
\begin{adjustbox}{max width=\linewidth}

\begin{tabular}{lcccccccc}
\toprule
& \multicolumn{2}{c}{LAION-AES $\uparrow$}
& \multicolumn{2}{c}{CLIPScore $\uparrow$}
& \multicolumn{2}{c}{CLIP-IQA+ $\uparrow$}
& \multicolumn{2}{c}{MUSIQ $\uparrow$} \\
\cmidrule(lr){2-3}
\cmidrule(lr){4-5}
\cmidrule(lr){6-7}
\cmidrule(lr){8-9}
Model
& Base & +Ours
& Base & +Ours
& Base & +Ours
& Base & +Ours \\
\midrule
SDXL
& 5.63 & 5.62
& 0.79 & 0.77
& 0.69 & 0.70
& 70.67 & 70.12 \\
FLUX
& 5.51 & 5.48
& 0.78 & 0.77
& 0.73 & 0.73
& 70.62 & 70.19 \\
SD3.5
& 5.33 & 5.28
& 0.81 & 0.79
& 0.67 & 0.65
& 69.82 & 69.53 \\
\bottomrule
\end{tabular}
\end{adjustbox}
\caption{General image quality comparison. Higher is better ($\uparrow$).}
\label{tab:appendix_quality}
\end{table*}

\section{Extended Analysis on SDXL Architecture}
\label{sec:analysis_sdxl}

To verify the generalizability of our motivation across different architectural paradigms, we extend our analysis to SDXL, a representative U-Net based Latent Diffusion Model. Figure~\ref{fig:analysis_sdxl} illustrates the spatiotemporal dynamics of cross-attention during multi-instance generation, corresponding to the analysis of SD3.5 presented in the main text.

\paragraph{Intensity Divergence.}
As shown in Figure~\ref{fig:analysis_sdxl}(a), the divergence in attention intensity between successfully generated (``Present'') and omitted (``Missing'') concepts remains evident. Consistent with our observations in SD3.5, ``Present'' concepts maintain higher mean attention scores throughout the diffusion process, whereas ``Missing'' concepts persistently exhibit lower intensity. This confirms that the failure to aggregate sufficient attention magnitude is a shared characteristic of concept omission in both U-Net and DiT architectures.

\paragraph{Early Detectability.}
Figure~\ref{fig:analysis_sdxl}(b) quantifies the detectability of omission failures using the AUC metric. Unlike the sharp peak observed in SD3.5, the AUC curve for SDXL demonstrates relative stability across the diffusion steps. Although the peak AUC value of approximately 0.64 is not necessarily attained at the very initial steps, the variation in AUC scores throughout the process is minimal. This suggests that while the peak discriminative signal is moderate compared to SD3.5, the reliability of detection in SDXL is sustained over a broader range of steps rather than being confined to a fleeting early phase.

\paragraph{Signal Stability.}
The analysis of spatial stability, measured by the Coefficient of Variation (CV), mirrors the findings in SD3.5. As depicted in Figure~\ref{fig:analysis_sdxl}(c), ``Present'' concepts are characterized by low CV values, indicating spatially concentrated and stable attention regions. In contrast, ``Missing'' concepts show persistently high CV values, reflecting scattered, noise-like attention distributions. This distinction highlights that regardless of the backbone architecture, concept omission is fundamentally linked to the inability to form a stable spatial representation in the latent space.

\section{Additional Experimental Results}
\paragraph{GenEval Results.}
We further evaluate Delta-K on GenEval, which assesses compositional generation across single-object, two-object, counting, color, position, and color-attribution settings. As shown in Table~\ref{tab:geneval}, Delta-K improves the overall performance of SDXL, with the most notable gains in two-object composition and color attribution, while preserving its single-object and positional performance, indicating that Delta-K strengthens challenging compositional relationships without degrading basic object generation or spatial fidelity.
\begin{table*}[t]
  \centering
\begin{adjustbox}{max width=\linewidth}

  \begin{tabular}{lccccccc}
    \toprule
    \multirow{2}{*}{Model} &
    \multirow{2}{*}{Overall$\uparrow$} &
    Single & Two &
    \multirow{2}{*}{Counting} &
    \multirow{2}{*}{Colors} &
    \multirow{2}{*}{Position} &
    Color \\
    & & object & object & & & & attribution \\
    \midrule
    Seedream 3.0~\cite{gao2025seedream} & 0.99 & 0.96 & 0.91 & 0.93 & 0.47 & 0.80 & 0.84 \\

    \midrule
    \multicolumn{8}{l}{\textit{Earlier Generative Models}} \\
    minDALL-E                 & 0.23 & 0.73 & 0.11 & 0.12 & 0.37 & 0.02 & 0.01 \\
    Stable Diffusion v2.1     & 0.50 & 0.98 & 0.51 & 0.44 & 0.85 & 0.07 & 0.17 \\

    \midrule

    Stable Diffusion XL & 0.55 & 0.98 & 0.74 & 0.39 & 0.85 & 0.15 & 0.23 \\

    \rowcolor{gray!15}
    +Delta-K (Ours) & 0.58 & 0.98 & 0.79 & 0.42 & 0.88 & 0.15 & 0.29 \\
    \rowcolor{gray!15}
     & \textcolor{darkgreen}{(+0.03)} & \textcolor{darkgreen}{(+0.00)} & \textcolor{darkgreen}{(+0.05)} & \textcolor{darkgreen}{(+0.03)} & \textcolor{darkgreen}{(+0.03)} & \textcolor{darkgreen}{(+0.00)} & \textcolor{darkgreen}{(+0.06)} \\

    \bottomrule
  \end{tabular}
\end{adjustbox}
    \caption{Experiments on Geneval. Higher is better ($\uparrow$).}
  \label{tab:geneval}
\end{table*}
\paragraph{ConceptMix Results.}
We also evaluate Delta-K on the challenging subsets of ConceptMix, where each prompt contains five to seven compositional elements. As shown in Table~\ref{tab:rank_k567}, Delta-K consistently improves the generation performance of both SDXL and FLUX-dev, demonstrating that the proposed key-space intervention generalizes across U-Net and DiT-based architectures. The improvements are more evident when multiple concepts compete during generation, while the relatively small absolute scores also reflect the difficulty of maintaining all specified concepts under highly complex prompts.

\begin{table}[htbp]
\centering
\begin{adjustbox}{max width=\linewidth}

\begin{tabular}{lccc}
\toprule
Model & $k{=}5\uparrow$ & $k{=}6\uparrow$ & $k{=}7\uparrow$ \\
\midrule
DALL-E
& 0.17 & 0.11 & 0.08 \\
SDXL
& 0.05 & 0.01 & 0.00 \\
\rowcolor{gray!15}
+Delta-K (Ours)
& 0.06 & 0.01 & 0.01 \\
\rowcolor{gray!15}
& \textcolor{darkgreen}{(+0.01)}
& \textcolor{darkgreen}{(+0.00)}
& \textcolor{darkgreen}{(+0.01)} \\
FLUX-dev
& 0.07 & 0.03 & 0.03 \\
\rowcolor{gray!15}
+Delta-K (Ours)
& 0.09 & 0.04 & 0.03 \\
\rowcolor{gray!15}
& \textcolor{darkgreen}{(+0.02)}
& \textcolor{darkgreen}{(+0.01)}
& \textcolor{darkgreen}{(+0.00)} \\
\bottomrule
\end{tabular}
\end{adjustbox}
\caption{Experiments on ConceptMix. Higher is better ($\uparrow$).}
\label{tab:rank_k567}
\end{table}

\paragraph{General Image Quality Evaluation.}
We evaluate whether Delta-K affects the general visual quality of generated images using LAION-AES~\citep{schuhmann2022laion}, CLIPScore~\citep{hessel2021clipscore}, CLIP-IQA+~\citep{wang2023exploring}, and MUSIQ~\citep{ke2021musiq}. As shown in Table~\ref{tab:appendix_quality}, the aesthetic scores remain comparable across SDXL, FLUX, and SD3.5, indicating that Delta-K preserves overall visual fidelity while improving compositional generation.

\section{Prompt for VLM}
\label{appendixb}

\begin{tcolorbox}[
    title=Prompts for Visual Fidelity Verification,
    colback=white,
    colframe=black,
    boxrule=0.5pt,
    arc=1mm,
    left=1mm,
    right=1mm,
    top=1mm,
    bottom=1mm,
    breakable,
    width=\linewidth
]

\# \textit{INPUT: Text Prompt, Generated Image}

You are a high-precision vision QA tool. Your task is to verify the visual fidelity of an image against a specific text prompt by focusing on concrete entities and their immediate modifiers.

\textbf{PROMPT:} \textit{\{prompt\}}

\textbf{Evaluation Steps:}
\begin{enumerate}[label=\arabic*), leftmargin=1.0cm,nosep]
\item \textbf{Entity-Attribute Extraction:}
Deconstruct the prompt into specific nouns or short \textit{modifier-noun} phrases (e.g., vintage clock, golden retriever, marble floor). Ignore broad actions or complex clauses.

\item \textbf{Visual Strictness Check:}
For each extracted phrase, verify whether the visual evidence exactly matches every word in that phrase. If the prompt specifies a transparent glass bottle but the image shows an opaque plastic bottle, the phrase should be marked as missing.

\item \textbf{Reasoning:}
Evaluate whether the specific physical attributes (color, texture, material, quantity, or relative position) described in the prompt are visibly evident in the image.
\end{enumerate}

\textbf{Task:}
\begin{enumerate}[label=\arabic*), leftmargin=1.0cm,nosep]
\item \textbf{present\_tokens:}
List up to top \textit{k} noun-based phrases or entities from the PROMPT where all modifiers are correctly and clearly depicted.

\item \textbf{missing\_tokens:}
List up to top \textit{k} noun-based phrases or entities from the PROMPT that are absent, have incorrect modifiers (e.g., wrong color or material), or are visually ambiguous.
\end{enumerate}

\textbf{Strict Rules:}
\begin{itemize}[label=-, leftmargin=0.8cm,nosep]
\item Extract only nouns or short adjective/attribute + noun structures.
\item Use the exact wording from the PROMPT. No paraphrasing.
\item Be strict with modifiers. If a ``blue silk tie'' appears as a ``blue wool tie'', it must be listed under missing\_tokens.
\item Return a strict JSON object only. No markdown or additional explanations.
\end{itemize}

\textbf{Output Format:}

{\ttfamily
\small
\{
"present\_tokens": [],
"missing\_tokens": []
\}
}

\end{tcolorbox}

\section{More Related Work}
\paragraph{Advanced Layout and Regional Control.} 
While explicit spatial adapters impose structural priors, recent advancements delve deeper into decomposing the multi-instance generation into localized sub-tasks~\citep{zhou2024migc,li2025mccd}or decoupling spatial planning from DiT-based detail rendering~\citep{zhou20243dis}. Further refinements align region-specific embeddings with orientation constraints~\citep{parihar2025compass}, high-frequency attention modulations, or dense typographic layouts~\citep{zhang2026freetext}, structurally enhancing the precision of multi-object positioning and attribute binding without demanding holistic retraining. Moreover, these spatial routing mechanisms have been fundamentally extended into 3D domains, enabling coherent multi-object scene generation via multi-instance diffusion~\citep{huang2025midi}.
\paragraph{LLM and VLM-Assisted Composition.} 
Beyond standard text encoders, incorporating Large Language Models (LLMs) as external cognitive planners significantly mitigates the omission of rare concepts~\citep{park2024rare}, initializes structured composite priors~\citep{khan2025composeanything}, and facilitates complex scene graph reasoning during inference~\citep{mishra2025compositional,peng2025ld}. Simultaneously, Vision-Language Models (VLMs) are increasingly integrated into the denoising trajectory as real-time semantic evaluators~\citep{lv2025multimodal}. This enables dynamic closed-loop optimizations—such as adaptive negative prompting~\citep{golan2025vlm} and iterative reward feedback learning—to ensure strict prompt-image alignment in dense visual scenes.
\paragraph{Deep Feature Modulation and Personalization.} 
Expanding upon superficial attention reweighting, emerging training-free paradigms actively intervene in deeper orthogonal feature representations, such as the Value space, for disentangled object-style blending~\citep{jin2026tp} and step-wise semantic injection~\citep{choi2026stepwise,cai2025ditctrl}. Particularly in multi-subject personalization tasks, establishing explicit semantic correspondence or employing layout-guided feature resamplers efficiently prevents catastrophic identity fusion. Recent efforts further eliminate the need for explicit spatial masks by leveraging self-supervised dual representation alignments or attention-based concept disentanglement~\citep{zhang2025conceptcraft}. These robust multi-instance identity preservation techniques are also being translated into spatiotemporal domains via intent-aware modulation and temporal-wise separable attention.
\section{More Examples of Delta-K}
We provide more examples of Delta-K comparing to baseline model and sota closed-source models in Figure~\ref{samples}.

\begin{figure*}[t]
    \centering
    \includegraphics[width=1\textwidth]{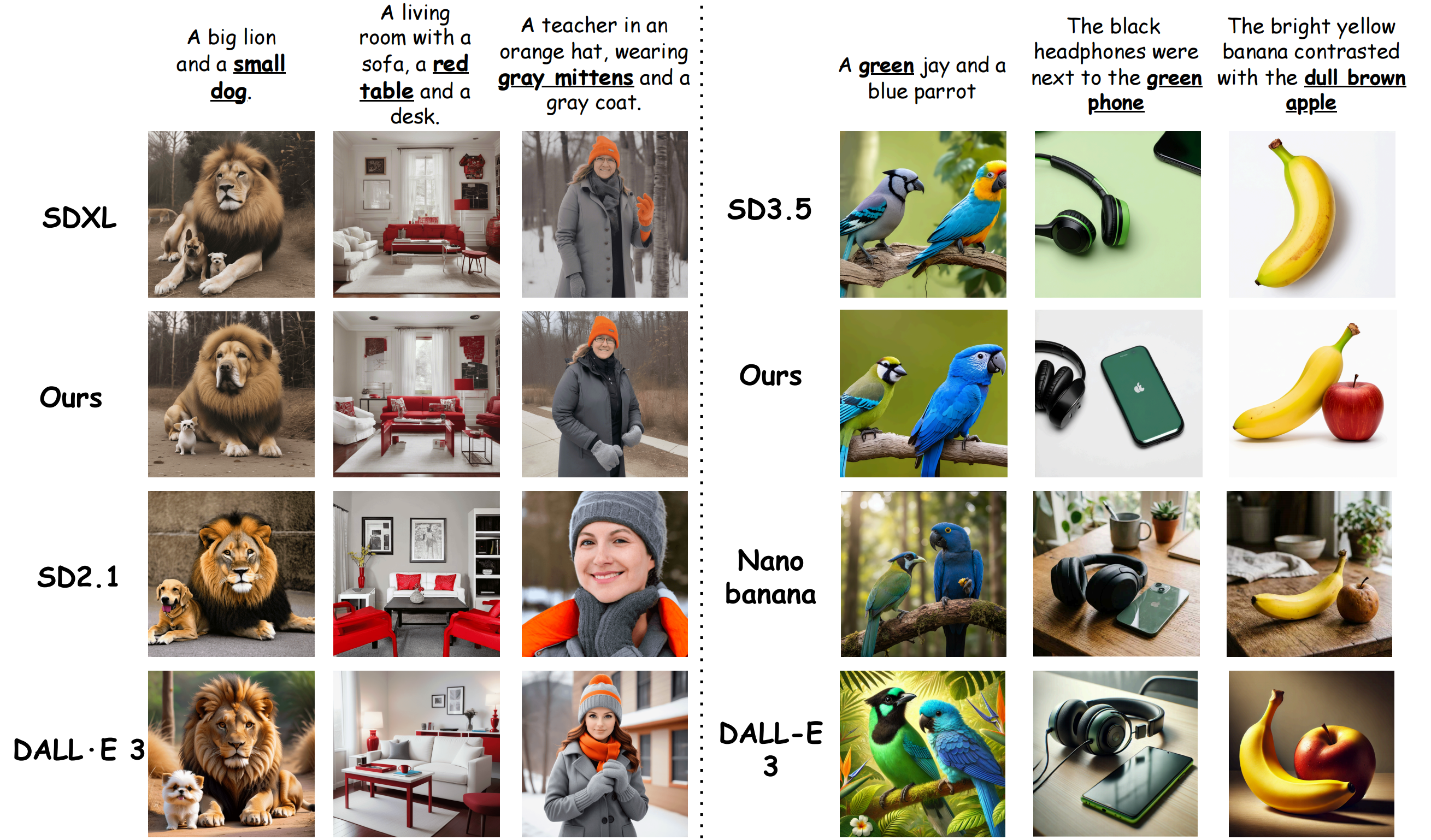}
    \caption{More examples of Delta-K compared with baseline and closed-source models.}
    \label{samples}
\end{figure*}